\documentclass{article}
\usepackage[final]{nips_2018}
\usepackage[utf8]{inputenc} % allow utf-8 input
\usepackage[T1]{fontenc}    % use 8-bit T1 fonts
\PassOptionsToPackage{hyphens}{url}\usepackage{hyperref}
\usepackage{url}            % simple URL typesetting
\usepackage{booktabs}       % professional-quality tables
\usepackage{nicefrac}       % compact symbols for 1/2, etc.
\usepackage{microtype}      % microtypography
\usepackage{geometry}
\usepackage{amsmath,amsfonts,amssymb,amsthm}
\usepackage{mathtools}
\usepackage{color}
\usepackage{xspace}
\usepackage{float}
\usepackage{color-edits}
\addauthor{jm}{red}
\usepackage{comment}
\usepackage{algorithmic}
\usepackage{natbib}
\usepackage{times}
\usepackage{graphicx}
\usepackage{enumitem}
\usepackage{bm}
\usepackage{xcolor}
\usepackage{tikz}
\usepackage{booktabs}
\usepackage[flushleft]{threeparttable}
\usepackage{tabularx}
\usepackage{subfig}
\usepackage{mleftright}
\usepackage[linesnumbered,ruled]{algorithm2e}

\bibpunct{[}{]}{;}{a}{,}{,}

\DeclareMathOperator*{\argmin}{arg\,min}

\newtheorem{fact}{Fact}[section]
\newtheorem{theorem}{Theorem}[section]
\newtheorem{lemma}[theorem]{Lemma}
\newtheorem{corollary}[theorem]{Corollary}

\newtheorem{definition}[theorem]{Definition} 

\newtheorem{remark}[theorem]{Remark}

\newenvironment{proofof}[1]{\smallskip\noindent{\bf Proof of #1}:}{$\hfill \Box$\\}

\newcommand{\cP}{\mathcal{P}}
\newcommand{\rank}{\mathsf{rank}}
\newcommand{\onevec}{\mathbf{1}}
\newcommand{\oracle}{\mathbf{oracle}}
\newcommand{\R}{\mathbb{R}}
\newcommand{\N}{\mathbb{N}}
\DeclareMathOperator{\Tr}{Tr}

\title{The Price of Fair PCA: One Extra Dimension}

\author{
   Samira Samadi \\
   Georgia Tech \\
   \texttt{ssamadi6@gatech.edu}\\
\And
Uthaipon Tantipongpipat \\
  Georgia Tech \\
  \texttt{tao@gatech.edu} \\
 \And
Jamie Morgenstern \\
  Georgia Tech \\
  \texttt{jamiemmt.cs@gatech.edu} \\
 \AND
 Mohit Singh \\
  Georgia Tech \\
  \texttt{mohitsinghr@gmail.com} \\
 \And
Santosh Vempala \\
 Georgia Tech \\
 \texttt{vempala@cc.gatech.edu} \\
}

\begin{document}

\maketitle

\begin{comment}
        -- Emphasize on marginal cost
        -- The algorithm is simply a few pcas
        -- gradient descend Ax multiplications  
\end{comment}

\begin{abstract}
  
  We investigate whether the standard dimensionality
  reduction technique of PCA inadvertently produces data
  representations with different fidelity for two different
  populations. We show on several real-world data sets, PCA has higher
  reconstruction error on population $A$ than 
  on 
  $B$ (for example, women
  versus men or lower- versus higher-educated individuals).
  This can happen even when the data set has a similar number of samples from $A$ and $B$.
  This motivates our study of dimensionality reduction
  techniques which maintain similar fidelity for $A$ and $B$. We define the notion of Fair PCA and give a
  polynomial-time algorithm for finding a low dimensional representation of the data which is
  nearly-optimal with respect to this measure. Finally, we show on real-world data sets that our algorithm can be used to  efficiently generate a fair low dimensional representation of the data. 
  \end{abstract}

\section{Introduction}

In recent years, the ML community has witnessed an onslaught of
charges that real-world machine learning algorithms have produced
``biased'' outcomes. The examples come from diverse and impactful
domains.  Google Photos labeled African Americans as
gorillas~\citep{twitter,simonite_2018} and returned queries for CEOs
with images overwhelmingly male and white~\citep{kay2015unequal},
searches for African American names caused the display of arrest
record advertisements with higher frequency than searches for white
names~\citep{Sweeney13}, facial recognition has wildly different
accuracy for white men than dark-skinned women~\citep{gendershades},
and recidivism prediction software has labeled low-risk African
Americans as high-risk at higher rates than low-risk white
people~\citep{propublica}.

The community's work to explain these observations has roughly fallen
into either ``biased data'' or ``biased algorithm'' bins.  In some cases, the
training data might under-represent (or over-represent) some group, or
have noisier labels for one population than another, or use an
imperfect proxy for the prediction label (e.g., using arrest records
in lieu of whether a crime was committed). Separately, issues of
imbalance and bias might occur due to an algorithm's behavior,
such as focusing on accuracy across the entire distribution rather
than guaranteeing similar false positive rates across populations, or
by improperly accounting for confirmation bias and feedback loops in
data collection. If an algorithm fails to distribute loans or bail to
a deserving population, the algorithm won't receive additional data
showing those people would have paid back the loan, but it will
continue to receive more data about the populations it (correctly)
believed should receive loans or bail.

Many of the proposed solutions to ``biased data'' problems amount to
re-weighting the training set or adding noise to some of the labels;
for ``biased algorithms'', most work has focused on maximizing
accuracy subject to a constraint forbidding (or penalizing) an unfair
model. Both of these concerns and approaches have significant merit,
but form an incomplete picture of the ML pipeline and where unfairness
might be introduced therein. Our work takes another step in fleshing
out this picture by analyzing when \emph{dimensionality reduction}
might inadvertently introduce bias. We focus on principal component
analysis (henceforth PCA), perhaps the most fundamental dimensionality
reduction technique in the sciences~\citep{kpfrs1901lines,hotelling1933analysis, jolliffe1986principal}.  We show several real-world data sets for which PCA
incurs much higher average reconstruction error for one population than
another, even when the populations are of similar
sizes. Figure~\ref{fig:wpca} shows that PCA on labeled faces in the
wild data set (LFW) has higher reconstruction error for women than men
even if male and female faces are sampled with equal weight.

\begin{figure*}[t]
\centering
  \includegraphics[width=0.47\linewidth, height = 4cm]{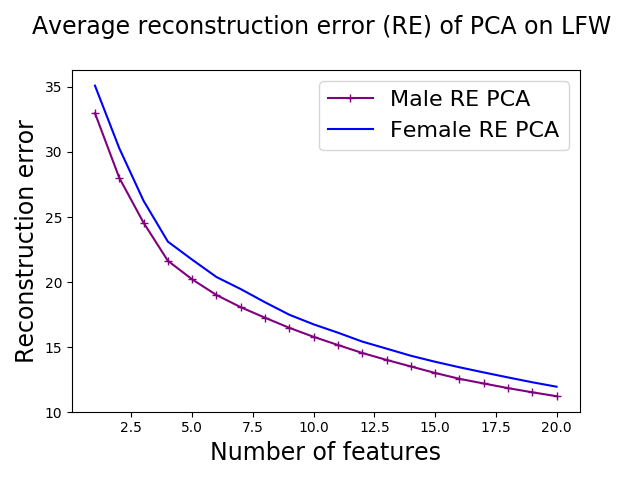}
  \hspace{.5cm}
  \includegraphics[width=0.47\linewidth, height = 4cm]{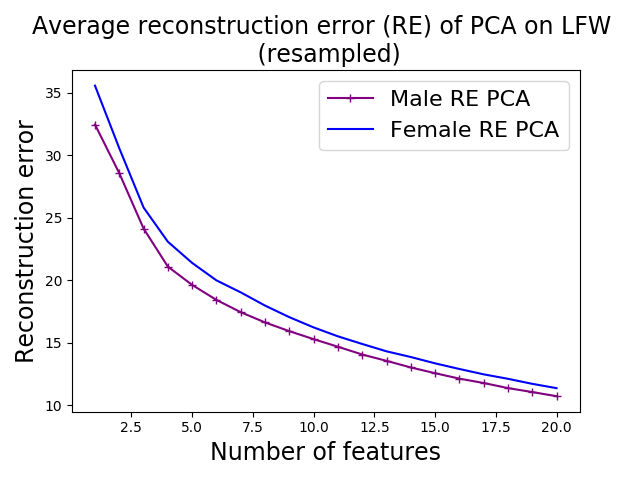}
  \caption{Left: Average reconstruction error of PCA on labeled faces in the wild data set (LFW), separated by
  gender. Right: The same, but sampling 1000 faces with men and women equiprobably 
  %(mean curves over sampling for 20 times).}
  (mean over 20 samples).}
  \label{fig:wpca}
\end{figure*}

This work underlines the importance of considering fairness and bias
at every stage of data science, not only in gathering and documenting
a data set~\citep{datasheets} and in training a model, but also in any
interim data processing steps. Many scientific disciplines have
adopted PCA as a default preprocessing step, both to avoid the curse
of dimensionality and also to do exploratory/explanatory data analysis
(projecting the data into a number of dimensions that humans can more
easily visualize). The study of human biology, disease, and the
development of health interventions all face both aforementioned
difficulties, as do numerous economic and financial analysis. In such
high-stakes settings, where statistical tools will help in making
decisions that affect a diverse set of people, we must take particular
care to ensure that we share the benefits of data science with a
diverse community.

We also emphasize this work has implications for representational
rather than just allocative harms, a distinction drawn by 
%Kate Crawford~
\citet{kate17} between how people are represented and what
goods or opportunities they receive. Showing primates in search
results for African Americans is repugnant primarily due to its
representing and reaffirming a racist painting of African Americans,
not because it directly reduces any one person's access to a
resource. If the default template for a data set begins with running
PCA, and PCA does a better job representing men than women, or white
people over minorities, the new representation of the data set itself
may rightly be considered an unacceptable sketch of the world it aims
to describe.

Our work proposes a different linear dimensionality reduction
which aims to represent two populations~$A$ and $B$ with similar
{\it fidelity}---which we formalize in terms of {\it reconstruction error}. Given an $n$-dimensional data set and its $d$-dimensional approximation, the reconstruction error of the data 
with respect to its low-dimensional approximation is the sum of squares of distances between the original data points and their approximated points in the $d$-dimensional subspace. To eliminate the effect of size of a population, we focus on  average reconstruction error over a population.  
One possible objective for our goal  
would find a $d$-dimensional approximation of the data
which minimizes the
maximum reconstruction error over the two populations. However, this
objective doesn't avoid grappling with the fact that population $A$
may perfectly embed into $d$ dimensions, whereas $B$ might require
many more dimensions to have low reconstruction error. 
In such cases, this objective would not necessarily favor a solution with average
reconstruction error of $\epsilon$ for $A$ and $y\gg\epsilon$ for $B$ over one with
$y$ error for $A$ and $y$ error for $B$. This holds even
if $B$ requires $y$ reconstruction error to be embedded into $d$~dimensions and thus the first solution is nearly optimal for both populations
in $d$ dimensions.

This motivates our focus on finding a projection which minimizes the
maximum \emph{additional} or \emph{marginal} reconstruction error for each population
above the optimal $n$ into $d$ projection for that population
alone. This quantity captures how much a population's reconstruction
error increases by including another population in the dimensionality
reduction optimization.  
Despite this computational problem appearing
more difficult than solving ``vanilla'' PCA, we introduce a
polynomial-time algorithm which finds an $n$ into $(d+1)$-dimensional
embedding with objective value better than any $d$-dimensional
embedding. 
Furthermore, we show that optimal solutions have equal additional
average error for populations $A$ and $B$.

\paragraph{Summary of our results}

We show PCA can overemphasize the reconstruction error for
one population over another (equally sized) population, and we should
therefore think carefully about dimensionality reduction in
domains where we care about fair treatment of different
populations. We propose a new dimensionality reduction problem which
focuses on representing $A$ and $B$ with similar additional
error over projecting $A$ or $B$ individually. We give a
polynomial-time algorithm which finds 
near-optimal solutions to this
problem. Our  algorithm relies on solving a semidefinite program (SDP),
which can be prohibitively slow for practical applications. We note 
that 
it
is possible to (approximately) solve an SDP with a much faster
multiplicative-weights style algorithm, whose running time in practice
is equivalent to solving standard PCA at most 10-15 times. The details
of the algorithm are given in the full version of this work. We then
evaluate the empirical performance of this algorithm on several
human-centric data sets.

\section{Related work}\label{sec:rw}

This work contributes to the area of fairness for machine learning
models, algorithms, and data representations.  One interpretation of
our work is that we suggest using Fair PCA, rather than PCA, when
creating a lower-dimensional representation of a data set for further
analysis. Both pieces of work which are most relevant to our work
take the posture of explicitly trying to reduce the correlation
between a sensitive attribute 
(such as race or gender)
and the new
representation of the data.  The first piece is a broad line of
work ~\citep{representation,Beutel2017DataDA,
  calmon2017optimized,madras18a,zhang2018mitigating} that aims to design representations
which will be conditionally independent of the protected attribute,
while retaining as much information as possible (and particularly
task-relevant information for some fixed classification
task). The second piece is the work by \citet{convexpca}, who also look to design PCA-like maps which
reduce the projected data's dependence on a sensitive attribute. Our
work has a qualitatively different goal: we aim not to hide a
sensitive attribute, but instead to maintain as much information about
each population after projecting the data.  In other words, we look
for representation with similar richness for population $A$ as $B$,
rather than making $A$ and $B$ indistinguishable.

Other work has developed techniques to obfuscate a sensitive attribute
directly~\citep{PedreshiRT08,KamiranCP10,CaldersV10,KamiranC11,LuongRT11,KamiranKZ12,KamishimaAAS12,HajianD13,FeldmanFMSV15,ZafarVGG15,FishKL16,AdlerFFRSSV16}.
 This line of work diverges
from ours in two ways. First, these works focus on representations
which obfuscate the sensitive attribute rather than a representation
with high fidelity regardless of the sensitive attribute. Second, most of
these works do not give formal guarantees on how much an objective
will degrade after their transformations. Our work directly minimizes
the amount by which each group's marginal reconstruction error increases.

Much of the other work on fairness for learning algorithms focuses on
fairness in classification or scoring~\citep{fta, hardt2016,
  kleinberg2016, chouldechova2017fair}, or online learning
settings~\citep{JKMR16, KKMPRVW17, EFNSV17, ensign2017runaway}. These
works focus on either statistical parity of the decision rule, or
equality of false positives or negatives, or an algorithm with a fair
decision rule. All of these notions are driven by a single learning
task rather than a generic transformation of a data set, while our work
focuses on a ubiquitous, task-agnostic preprocessing step.

\section{Notation and vanilla PCA}
We are given $n$-dimensional data points represented as rows of matrix
$M \in \R^{m \times n}$. We will refer to the {\it set} and {\it
  matrix} representation interchangeably.  The data consists of two 
subpopulations $A$ and $B$ corresponding to two groups with different value of 
a binary sensitive attribute (e.g., males and females).  We denote by
\(\left[ \begin{array}{c} A \\ B \end{array} \right] \) the
concatenation of two matrices \(A,B\) by row. We refer to the $i^{th}$
row of $M$ as $M_i$, the $j^{th}$ column of $M$ as $M^j$ and the
$(i,j)^{th}$ element of $M$ as $M_{ij}$. We denote the Frobenius norm
of matrix $M$ by $\|M\|_F$ and the $2$-norm of the vector $M_i$ by
$\|M_i\|$. For 
$k\in \N$, we write $[k]:=\{1, \ldots, k\}$.
$|A|$
denotes the  size of a set $A$. Given two matrices $M$ and $N$ of the same
size, the Frobenius inner product of these matrices is defined as
$\langle M,N \rangle= \sum_{ij} M_{ij}N_{ij} = \Tr(M^TN)$.

%%%%%%%%%%%%%%%%%%%%%%%%%%%%%%%%%%%%%%%%%%%%%%%%%%%%%%%%%%%%%%%%%%%%%%%%%%%%%%%%%%%%%%%%%%%%%%%%%%%%%%%%%%%%%%%%%
%%%%%%%%%%%%%%%%%%%%%%%%%%%%%%%%%%%%%%%%%%%%%%%%%%%%%%%%%%%%%%%%%%%%%%%%%%%%%%%%%%%%%%%%%%%%%%%%%%%%%%%%%%%%%%%%%

\subsection{PCA}
This section recalls 
%several 
useful facts about PCA that we use in
later sections. We begin with a reminder of the definition of the PCA problem in
terms of minimizing the reconstruction error of a data set.

\begin{definition} (PCA problem) \label{def:fair-PCA}
 Given a matrix
  $M \in \mathbb{R}^{m \times n}$, find a matrix
  $\widehat{M} \in \R^{m \times n}$ of rank at most $d$ $(d\leq n)$
  that minimizes $\| M-\widehat{M} \|_F$.
\end{definition}

We will refer to $\widehat{M}$ as an optimal rank-$d$ approximation
of $M$. The following well-known fact characterizes the solutions to this classic problem 
[e.g., \citealp{shalev2014understanding}].

\begin{fact}
\label{lem:optPCA}
If $\widehat{M}$ is a solution to the PCA problem, then $\widehat{M} = MWW^T$ for a
matrix $W \in \R^{n\times d}$ with $W^TW = I$. 
The columns of  $W$ are eigenvectors corresponding to top $d$ eigenvalues of  
$M^TM$.
\end{fact}

The matrix $WW^T\in \R^{n\times n}$ is called a projection matrix.

%%%%%%%%%%%%%%%%%%%%%%%%%%%%%%%%%%%%%%%%%%%%%%%%%%%%%%%%%%%%%%%%%%%%%%%%%%%%%%%%%%%%%%%%%%%%%%%%%%%%%%%%%%%%%%%%%
%%%%%%%%%%%%%%%%%%%%%%%%%%%%%%%%%%%%%%%%%%%%%%%%%%%%%%%%%%%%%%%%%%%%%%%%%%%%%%%%%%%%%%%%%%%%%%%%%%%%%%%%%%%%%%%%%

\section{Fair PCA}

Given the $n$-dimensional data with two subgroups $A$ and $B$, let
$\widehat{M}, \widehat{A}, \widehat{B}$ be optimal rank-$d$ PCA
approximations for $M, A,$ and $B$, respectively.  We introduce our
approach to fair dimensionality reduction by giving two compelling
examples of settings where dimensionality reduction inherently makes a
tradeoff between groups $A$ and $B$. Figure~\ref{fig:ex1} shows a
setting where projecting onto any single dimension either favors $A$
or $B$ (or incurs significant reconstruction error for both), while
either group separately would have a high-fidelity embedding into a
single dimension. This example suggests any projection will
necessarily make a trade off between error on $A$ and error on $B$.

Our second example (shown in Figure~\ref{fig:ex2}) exhibits a setting
where $A$ and $B$ suffer very different reconstruction error when
projected onto one dimension: $A$ has high reconstruction error for
every projection while $B$ has a perfect representation in the
horizontal direction. Thus, asking for a projection which 
minimizes the maximum
reconstruction error for groups $A$ and $B$ might require
incurring additional error for $B$ while not improving the error for
$A$. So, 
minimizing the maximum reconstruction error over $A$ and $B$ fails to
account for the fact that two populations might have wildly different
representation error when embedded into $d$ dimensions. Optimal
solutions to such objective might behave in a counterintuitive way, preferring to
exactly optimize for the group with larger inherent representation
error rather than approximately optimizing for both groups
simultaneously. We find this behaviour undesirable---it requires
sacrifice in quality for one group for no improvement for
the other group.
 \begin{figure}[!tbp]
  \centering
  \subfloat[The best one dimensional PCA projection for group $A$ is vector $(1,0)$ and for group $B$ it is vector~$(0,1)$.]{\includegraphics[width=0.45\textwidth]{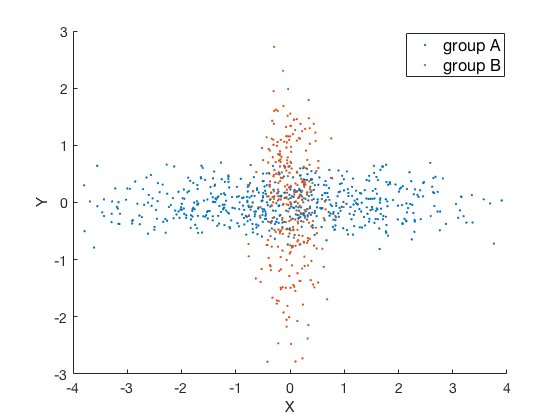}\label{fig:ex1}}
  \hfill
  \subfloat[Group $B$ has a perfect one-dimensional projection. For group $A$, any one-dimensional projection is equally bad.]{\includegraphics[width=0.45\textwidth]{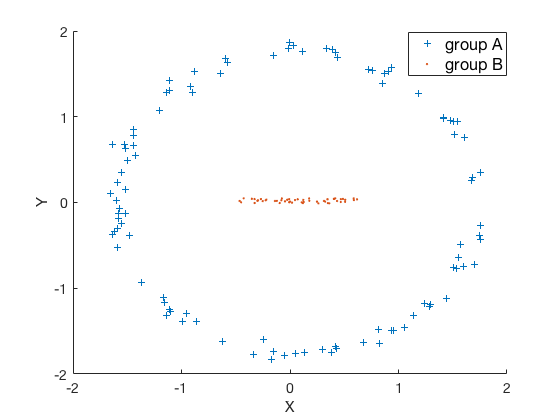}\label{fig:ex2}}
  \caption{}
\end{figure}

\begin{remark}
  We focus on the setting where we ask for a single projection into
  $d$ dimensions rather than two separate projections because using
  two distinct projections (or more generally two models) for
  different populations raises legal and ethical concerns. Learning
  two different projections also faces no inherent tradeoff in
  representing $A$ or $B$ with those projections.\footnote{\citet{lipton}  
    has asked whether equal treatment requires different models for
    two groups.}
  \end{remark}

  We therefore turn to finding a projection which minimizes the
  maximum deviation of each group from its optimal projection. This
  optimization asks that $A$ and $B$ suffer a similar \emph{loss} for
  being projected together into $d$ dimensions compared to their
  individually optimal projections. We now introduce our notation for
  measuring a group's loss when being projected to $Z$ rather than to its
  optimal $d$-dimensional representation:

  \begin{definition}[Reconstruction error] Given two matrices $Y$ and
    $Z$ of the same size, the reconstruction error of $Y$ with
    respect to $Z$ is defined as
$$\text{error}(Y, Z) = \|Y - Z\|_F^2.$$
\end{definition}

\begin{definition}[Reconstruction loss] Given a matrix $Y\in \mathbb{R}^{a\times n}$, let
  $\widehat{Y}\in \mathbb{R}^{a \times n}$ be the optimal rank-$d$ approximation of $Y$. For
  a matrix $Z\in \mathbb{R}^{a\times n}$ with rank at most $d$
we define
        \begin{align*}
        loss(Y,Z) & : = \|Y -Z\|_F^2 - \|Y - \widehat{Y}\|_F^2.
        \end{align*}
      \end{definition}

      Then, the optimization that we study asks to minimize the maximum
      loss suffered by any group. This captures the idea that,
      fixing a feasible solution, the objective will only improve if
      it improves the loss for the group whose current representation
      is worse. Furthermore, considering the reconstruction loss and not the reconstruction error prevents the optimization from incurring error for one subpopulation without improving the error for the other one as described in Figure~\ref{fig:ex2}.

\begin{definition}[Fair PCA]
  \label{def:fairPCA}
  Given $m$ data points in $\mathbb{R}^n$ with subgroups $A$ and $B$, we define the
  problem of finding a fair PCA projection into $d$-dimensions as
  optimizing
 \begin{align}
 \label{minimizationMAIN}
        \min_{U\in \R^{m\times n} \text{, } \rank(U) \leq  d} \max  & \left\{ \frac{1}{|A|} loss(A,U_A),  \frac{1}{|B|}  loss (B, U_B)\right\},
 \end{align}
 where $U_A$ and $U_B$ are matrices with rows corresponding to rows of 
 $U$ for groups $A$ and $B$ respectively.
\end{definition}

This definition does not appear to have a closed-form solution (unlike
vanilla PCA---see Fact \ref{lem:optPCA}). To take a step in characterizing solutions to this
optimization, Theorem~\ref{thm:sameLoss} states that a fair PCA
low dimensional approximation of the data results in the same loss for both groups.

\begin{theorem}
\label{thm:sameLoss}
Let $U$ be a solution to the Fair PCA problem \eqref{minimizationMAIN}, then
\[
\frac{1}{|A|} loss(A, U_A)= \frac{1}{|B|} loss(B, U_B).
\]
\end{theorem}

Before proving Theorem~\ref{thm:sameLoss}, we need to state some 
building blocks of the proof, Lemmas~\ref{lem:A}, \ref{lem:B}, and \ref{lem:C}. For
the proofs of the lemmas please refer to the appendix~\ref{app:proof}.

\begin{lemma}
\label{lem:A}
Given a matrix $U\in \mathbb{R}^{m\times n}$ such that
$\rank(U)\leq d$ , let
$f(U)=\max \left\{ \frac{1}{|A|} loss(A,U_A), \frac{1}{|B|} loss (B,
U_B)\right\}$. Let $\{v_1,\ldots, v_d\} \subset \R^n $ be an orthonormal
basis of the row space of $U$ and
$V := [v_1, \ldots , v_d] \in \mathbb{R}^{n \times d}$. Then
  $$
f\left(\left[ \begin{array}{c} A \\ B \end{array} \right] VV^T\right) = f\left(\left[ \begin{array}{c} AVV^T \\ BVV^T \end{array} \right]\right) \leq f(U).
  $$
\end{lemma}
The next lemma presents some equalities that we will use frequently in the proofs.
\begin{lemma}
\label{lem:B}
        Given a matrix $V=[v_1,\ldots, v_d]\in \mathbb{R}^{n\times d}$ with orthonormal columns, we have:
        \begin{itemize}
                \item[$\diamond$] $loss(A, AVV^T) = \|\widehat{A}\|_F^2 - \sum_{i=1}^{d} \|Av_i\|^2 = \|\widehat{A}\|_F^2 - \langle A^TA, VV^T\rangle$
                \item[$\diamond$] $\|A-AVV^T\|_F^2 =  \|A\|_F^2 - \|AV\|_F^2 = \|A\|_F^2 - \sum_{i=1}^d \|Av_i\|^2$
        \end{itemize}
\end{lemma}
Let the function $g_A=g_A(U)$ measure the reconstruction error of a fixed matrix $A$ with respect to its orthogonal projection to the input
subspace $U$. The next lemma shows that the value of the function~$g_A$ at any local minimum is the same.
\begin{lemma}
\label{lem:C}
Given a matrix $A\in \mathbb{R}^{a\times n}$, and a $d$-dimensional subspace $U$,
let the function $g_A=g_A(U)$ denote the reconstruction error of matrix
$A$ with respect to its orthogonal projection to the subspace~$U$, that is  $g_A(U) := \|A-AUU^T\|_F^2$, where by abuse of notation we use
$U$ inside the norm to denote the matrix which has an orthonormal
basis of the subspace $U$ as its columns.
The value of the function~$g_A$ at any local minimum is the same. 

\end{lemma}

\begin{proofof}{Theorem~\ref{thm:sameLoss}}

  Consider the functions $g_A$ and $g_B$ defined in Lemma~\ref{lem:C}.
  It follows from Lemma~\ref{lem:A} and Lemma~\ref{lem:B} that for
  $V\in \mathbb{R}^{n\times d}$ with  $V^TV=I$ we have
        \begin{align}
        \label{eq:lossEq}
        loss(A, AVV^T) = \|\widehat{A}\|_F^2 - \|A\|_F^2 + g_A(V), \\
        loss(B, BVV^T) = \|\widehat{B}\|_F^2 - \|B\|_F^2 + g_B(V).  \nonumber       
        \end{align}

Therefore, the Fair PCA problem is equivalent to
\begin{align*}
  \min_{V \in \mathbb{R}^{n\times d}, V^TV=I} \, f(V):=\max \left\{ \frac{1}{|A|} loss(A, AVV^T)
  , \frac{1}{|B|} loss(B, BVV^T) \right\}.
\end{align*}

We proceed to prove the claim by contradiction.
Let $W$ be a global minimum of $f$ and assume that
\begin{equation}
\label{eq:bigger}
\frac{1}{|A|}loss(A, AWW^T) > \frac{1}{|B|}loss(B, BWW^T).
\end{equation}

Hence, since $loss$ is continuous, for any matrix $W_\epsilon$ with $W_\epsilon ^T W_\epsilon =I$ in a small enough neighborhood  of $W$, $f(W_\epsilon)= \frac{1}{|A|}loss(A, AW_\epsilon W_\epsilon ^T)$. Since $W$ is a global minimum of $f$, it is a local minimum of $\frac{1}{|A|} loss(A, AWW^T)$ or equivalently a local minimum of 
$g_A$ because of \eqref{eq:lossEq}.

Let $\{v_1, \ldots, v_n\}$ be an orthonormal basis of the eigenvectors
of $A^TA$ corresponding to eigenvalues
$\lambda_1 \geq \lambda_2 \geq \ldots \geq \lambda_n$. Let $V^*$ be
the subspace spanned by $\{v_1,\ldots,v_d\}$. Note that
$loss(A, A{V^*}^TV^*)=0$. Since the loss is always non-negative for
both $A$ and $B$, \eqref{eq:bigger} implies that $loss(A, AWW^T) > 0$.
Therefore, $W\neq V^*$ and $g_A(V^*) < g_A(W)$. By Lemma~\ref{lem:C}, this is in contradiction with $V^*$ being a global minimum and $W$ being a local minimum of $g_A$.
\end{proofof}

\section{Algorithm and analysis}

In this section, we present a polynomial-time algorithm for solving the fair PCA problem. Our algorithm outputs a matrix of rank at most $d+1$ and guarantees that it achieves the fair PCA objective value equal to the optimal
$d$-dimensional fair PCA value.  The algorithm
has two steps: first, relax fair PCA to a semidefinite optimization
problem and solve the SDP; second, solve an LP designed to reduce the rank of said solution. We
argue using properties of extreme point solutions that the solution must satisfy a
number of constraints of the LP with equality, and argue directly that this
implies the solution must lie in $d+1$ or fewer dimensions. We refer the reader to ~\cite{lau2011iterative} for basics and applications of this technique in approximation algorithms.

\begin{theorem}
\label{thm:main}
There is a polynomial-time algorithm that outputs an approximation matrix of the data such that it is either of rank $d$ and is an optimal solution to the fair PCA problem OR it is of rank $d+1$, has equal losses for the two populations and achieves the optimal fair PCA objective value for dimension $d$.
\end{theorem}

\begin{algorithm} % enter the algorithm environment
\caption{Fair PCA} % give the algorithm a caption
\label{alg:fairPCA} % and a label for \ref{} commands later in the document
    \SetKwInOut{Input}{Input}
    \SetKwInOut{Output}{Output}
    \Input{\(A\in\R^{m_1\times n} ,B\in\R^{m_2\times n} \), \(d < n,m=m_1+m_2 \)}
    \Output{\({U \in \mathbb{R}^{m \times n}, \rank(U)\leq d+1}\)}
    Find optimal rank-\(d\) approximations of \(A,B\) as \(\widehat{A},\widehat{B}\) (e.g. by Singular Value Decomposition).   

Let ($\hat{P},\hat{z})$ be a solution to the SDP:\begin{align}
\min\nolimits_{P\in \mathbb{R}^{n\times n}\text{, }z\in\R} &\ \ \ z \label{sdp}\\
\text { s.t. } & z \geq \frac{1}{m_1} \cdot \left(  \|\widehat{A}\|_F^2 -\langle A^\top A,P\rangle \right) \nonumber \\
& z\geq \frac{1}{m_2} \cdot \left(  \|\widehat{B}\|_F^2 - \langle B^\top B,P\rangle \right) \nonumber \\
&  \text{Tr}(P) \leq d , \ 0 \preceq P \preceq I \nonumber
\end{align}  \\
    Apply Singular Value Decomposition to \(\hat{P}\), \(
\hat{P} = \sum_{j=1}^n \hat{\lambda}_ju_j u_j^T\).
     \\
     Find an extreme solution \((\bar{\lambda},z^*)\) of the LP: \begin{align}
     \min_{\lambda  \in \mathbb{R}^{n }\text{, }z\in\R} &\ \ \ z \label{eq:two} \\
\text { s.t. }  z &\geq\frac{1}{m_1}  \Big(   \|\widehat{A}\|_F^2 - \langle A^\top A,\sum_{j=1}^n \lambda_j u_j u_j^T\rangle \Big)=\frac{1}{m_1}  \Big(  \|\widehat{A}\|_F^2 - \sum_{j=1}^n \lambda_j \cdot \langle A^\top A,u_j u_j^T\rangle\Big) \label{eq:con1}\\
z &\geq\frac{1}{m_2}  \Big(   \|\widehat{B}\|_F^2 - \langle B^\top B,\sum_{j=1}^n \lambda_j u_j u_j^T\rangle\Big)= \frac{1}{m_2}  \Big( \|\widehat{B}\|_F^2 - \sum_{j=1}^n \lambda_j \cdot \langle B^\top B,u_j u_j^T \rangle\Big) \\
  & \textstyle{\sum_{i=1}^n\lambda_i \leq d} \label{eq:three}\\
  & 0 \leq \lambda_i \leq 1 \label{eq:four}
  \end{align}
  
     Set  $
P^*= \sum_{j=1}^n \lambda^*_j u_j u_j^T
$ where \(\lambda_j^* = 1-\sqrt{1-\bar{\lambda_j}}\). \\
    \Return
%     \(U =
% \mleft[
% \begin{array}{c}
%    A\\
%   \hline
%   B
% \end{array}
% \mright]P^* \)
\(U =\left[ \begin{array}{c} A \\ B \end{array} \right] P^* \)

\end{algorithm}
\vspace{-2mm}

\begin{proofof}{Theorem~\ref{thm:main}}
The algorithm to prove Theorem  \ref{thm:main} is presented in Algorithm~\ref{alg:fairPCA}. Using Lemma~\ref{lem:B}, we can write the semi-definite relaxation of the fair PCA objective (Def.~\ref{def:fairPCA}) as SDP \eqref{sdp}. This semi-definite program can be solved in polynomial time.
The system of constraints~\eqref{eq:two}-\eqref{eq:four} is a linear program in the variables $\lambda_i$ (with the $u_i$'s fixed). Therefore, an extreme point solution \((\bar{\lambda},z^*)\) is defined by $n+1$ equalities, at most  three of which can be constraints in \eqref{eq:con1}-\eqref{eq:three} and the rest (at least $n-2$ of them) must be from the $\bar{\lambda}_i = 0$ or $\bar{\lambda}_i = 1$ for $i\in[n]$. Given the upper bound of $d$ on the sum of the $\bar{\lambda}_i$'s, this implies that at least $d-1$ of them are equal to $1$, i.e., at most two are fractional and add up to $1$.
\vspace{-3 mm}
\paragraph{Case 1.} All the eigenvalues are integral. Therefore, there are
  $d$ eigenvalues equal to $1$. This results in orthogonal projection
  to $d$-dimension.
\paragraph{Case 2.} $n-2$ of eigenvalues are in $\{0,1\}$ and two
  eigenvalues $ 0 < \bar{\lambda}_d, \bar{\lambda}_{d+1} < 1$. Since
  we have $n+1$ tight constraints, this means that both of the first two
  constraints are tight. Therefore
  \vspace{-2mm}
  $$
  \frac{1}{|A|}(\|\widehat{A}\|_F^2 - \sum_{i=1}^{n}\bar{\lambda}_i
  \langle {A^TA,u_i u_i^T} \rangle ) =
  \frac{1}{|B|}(\|\widehat{B}\|_F^2 - \sum_{i=1}^{n}\bar{\lambda}_i
  \langle {B^TB,u_i u_i^T} \rangle)=z^*\leq \hat{z},
  $$
  where the inequality is by observing that \((\hat{\lambda},\hat{z})\) is a feasible solution.  Note that the loss of group \(A\) given by an affine projection \(P^*= \sum_{j=1}^n \lambda^* u_j u_j^T\) is
\begin{align*}
loss(A,AP^*)  = \|{A-AP^*}\|_F^2 - \|A- \widehat{A}\|^2_F = \text{Tr}\left((A-AP^*)(A-AP^*)^\top \right) - \|{A}\|_F^2 +\| \widehat{A}\|_F^2 \\
= \text{Tr}\left((A-AP^*)(A-AP^*)^\top \right) - \|{A}\|_F^2 +\| \widehat{A}\|_F^2 = \| \widehat{A}\|_F^2 - 2\text{Tr}(AP^*A^\top)+\text{Tr}(A{P^*}^2A^\top)\\
=\|\widehat{A}\|_F^2 - \sum_{i=1}^{n}(2\lambda_i^* -{\lambda_i^*}^2)\langle {A^TA,u_i u_i^T} \rangle =\|\widehat{A}\|_F^2 - \sum_{i=1}^{n}\bar{\lambda}\langle {A^TA,u_i u_i^T} \rangle  , \\
\end{align*}
where the last inequality is by the choice of \(\lambda_j^* = 1-\sqrt{1-\bar{\lambda_j}}\). The same equality holds true for group \(B\). Therefore, \(P^*\) gives the equal loss of \(z^*\leq\hat{z}\) for two groups.
The embedding
$x \rightarrow (x\cdot u_1, \ldots, x\cdot u_{d-1}, \sqrt{\lambda_d^*}\, x\cdot u_d, \sqrt{\lambda_{d+1}^*} \, x \cdot u_{d+1})
$
corresponds to the affine projection of any point (row) of $A,B$ defined by the solution $P^*$.

In both cases, the objective value is at most that of the original fairness objective.
\end{proofof}

The result of Theorem  \ref{thm:main} in two groups generalizes to more than two groups as follows.
Given $m$ data points in $\mathbb{R}^n$ with \(k\) subgroups $A_1,A_2,\ldots,A_k$, and \(d\leq n \) the desired number of dimensions of projected space, we generalize Definition \ref{def:fairPCA} of fair PCA
  problem  as
  optimizing
 \begin{align}
 \label{minimization}
        \min_{U\in \R^{m\times n} \text{, } \rank(U) \leq  d} \ \max_{i\in\{1,\ldots,k\}}  & \left\{ \frac{1}{|A_i|} loss(A_i,U_{A_i}))\right\},
 \end{align}
 where $U_{A_i}$ are  matrices with rows corresponding to rows of
 $U$ for groups $A_i$.

\begin{theorem}
\label{thm:k-group}
There is a polynomial-time algorithm to find a projection such that it is of dimension at most $d+k-1$ and achieves  the optimal fairness objective value for dimension $d$.
\end{theorem}
In contrast to the case of two groups, when there are more than two groups in the data, it is possible that all optimal solutions to fair
PCA will not assign the same loss
to all groups. However, with \(k-1\) extra dimensions, we can ensure
that the loss of each group remains at most the optimal fairness
objective in \(d\) dimension. The result of Theorem~\ref{thm:k-group}
follows by extending algorithm in Theorem~\ref{thm:main} by adding
linear constraints to SDP and LP for each extra group. An extreme
solution \((\bar{\lambda},z^*)\) of the resulting LP contains at most
\(k\) of \(\lambda_i\)'s that are strictly in between 0 and
1. Therefore, the final projection matrix \(P^*\) has rank at most
\(d+k-1\).

\paragraph{Runtime} We now analyze the runtime of Algorithm~\ref{alg:fairPCA}, which consists of solving SDP \eqref{sdp} and finding an extreme solution to an LP~\eqref{eq:two}-\eqref{eq:four}. The SDP and LP can be solved up to additive error of \(\epsilon>0\) in the objective value   in \(O(n^{6.5}\log(1/\epsilon))\)  \citep{ben2001lectures} and \(O(n^{3.5}\log(1/\epsilon)) \)  \citep{schrijver1998theory} time, respectively. The running time of SDP dominates the algorithm both in theory and practice, and is too slow for practical\ uses for moderate size of \(n\).

We propose another algorithm of solving SDP using the multiplicative
weight (MW) update method. In theory, our MW takes
\(O(\frac{1}{\epsilon^2})\) iterations of solving standard PCA, giving a total of\ \(O(\frac{n^3}{\epsilon^2})\) runtime, which may or may not be faster than
\(O(n^{6.5}\log(1/\epsilon))\) depending on
\(n,\epsilon\). In practice, however, we observe that after appropriately tuning one parameter in MW, the MW\ algorithm achieves accuracy \(\epsilon<10^{-5}\) within tens of iterations, and therefore is used to obtain experimental results in this paper. Our MW can handle data of
dimension up to a thousand with running time in less than a
minute. The details of implementation and analysis of MW method are in
Appendix \ref{sec:mw}.

\section{Experiments}

We use two common human-centric data sets for our experiments.  The
first one is labeled faces in the wild (LFW)~\citep{LFWTech}, the
second is the Default Credit data set~\citep{default-dataset}. We preprocess all data to have its mean at the origin. For the LFW
data, we normalized each pixel value by $\tfrac{1}{255}$.  The gender
information for LFW was taken from~\citet{afifi2017afif4}, who
manually verified the correctness of these labels. For the credit data, since different attributes are measurements of
incomparable units, we normalized the variance of each attribute to be
equal to 1. The code of all experiments is publicly available at \url{https://github.com/samirasamadi/Fair-PCA}.

\paragraph{Results}

We focus on projections into relatively few dimensions,
as those are used ubiquitously in early phases of data exploration. 
As we already saw in Figure~\ref{fig:wpca} left, at lower dimensions,
there is a noticeable gap between PCA's average reconstruction error for men and women on the LFW data set. This gap is at the
scale of up to 10\% of the total reconstruction error when we project to 20 dimensions. This still holds when we subsample male and female
faces with equal probability from the data set, and so men and women
have equal magnitude in the objective function of PCA (Figure~\ref{fig:wpca} right). 

Figure~\ref{fig:re} shows the average reconstruction error of each population (Male/Female, Higher/Lower education) as the result of running vanilla PCA and Fair PCA on LFW and Credit data. 
%The main message of this figure is the natural observation that 
As we expect, as the number of dimensions increase, the average reconstruction error of every population decreases. For LFW, the original data is in 1764 dimensions (42$\times$42 images), therefore, at 20 dimensions we still see a considerable reconstruction error. For the Credit data, we see that at 21 dimensions, the average reconstruction error of both populations reach 0, as this data originally lies in 21 dimensions. In order to see how fair are each of these methods, we need to zoom in further and look at the average loss of populations.  

Figure~\ref{fig:loss} shows the average loss of each population as the result of applying vanilla PCA and Fair PCA on both data sets.
%Figure~\ref{fig:loss} shows the results of running fair PCA on both data sets. We run fair PCA on LFW, and compare the fair loss of fair PCA to
%``vanilla'' PCA. 
%We break the loss down by gender for PCA (for fair
Note that at the optimal solution of Fair PCA, the average loss of two populations are the same, therefore we have one line for ``Fair loss''. We observe that PCA suffers
much higher average loss for female faces than male faces. After running fair PCA, we observe that the average loss for fair PCA is relatively in the middle of the average loss for male and female. So, there is improvement in
terms of the female average loss which comes with a
cost in terms of male average loss. Similar observation holds for the Credit data set. 
%Similarly, on the Credit data set, the average loss of fair PCA is
%somewhere between PCA's higher education average loss and lower education
%average loss.
In this context, it appears there is some cost to optimizing for
the less well represented population in terms of the
better-represented population.

\begin{figure}
\includegraphics[width=0.45\textwidth]{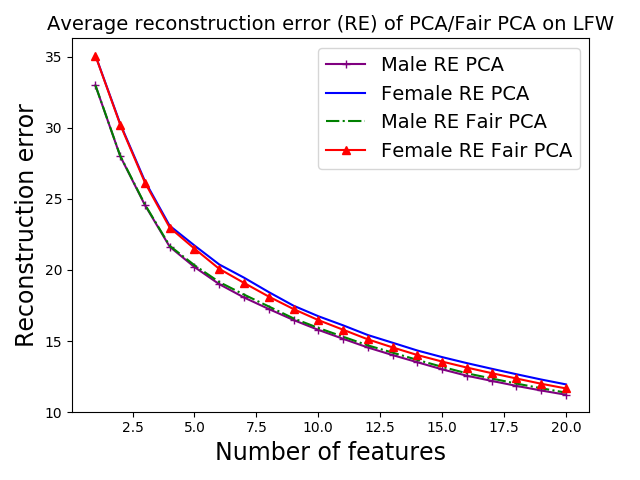}
\includegraphics[width=0.45\textwidth]{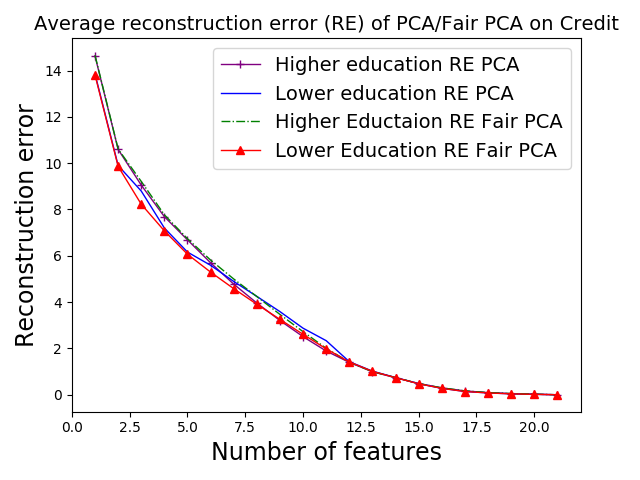}
\caption{Reconstruction error of PCA/Fair PCA on LFW and the Default Credit data set.}\label{fig:re}
\end{figure}

\begin{figure}
\includegraphics[width=0.45\textwidth]{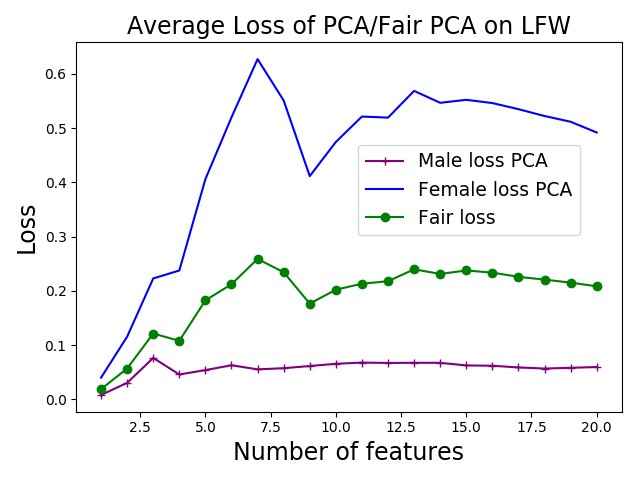}
\includegraphics[width=0.45\textwidth]{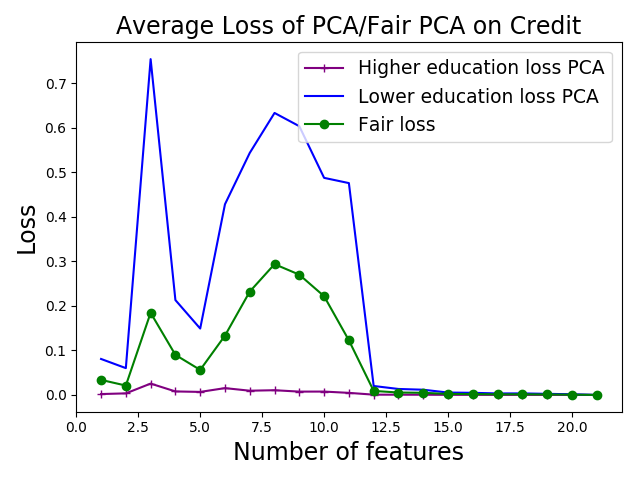}
\caption{Loss of PCA/Fair PCA on LFW and the Default Credit data set.}\label{fig:loss}
\end{figure}

\section{Future work}

This work is far from a complete study of when and how dimensionality
reduction might help or hurt the fair treatment of different
populations.  Several concrete theoretical questions remain using our
framework. What is the complexity of optimizing the fairness
objective? Is it NP-hard, even for $d=1$? Our work naturally extends
to $k$ predefined subgroups rather than just $2$, where the number of
additional dimensions our algorithm uses is $k-1$. Are these
additional dimensions necessary for computational efficiency?

In a broader sense, this work aims to point out another way in which
standard ML techniques might introduce unfair treatment of some
subpopulation. Further work in this vein will likely prove very
enlightening.

\newpage
\section*{Acknowledgements} 
This work was supported in part by NSF awards CCF-1563838, CCF-1717349, and CCF-1717947.

\bibliography{ref,agnostic-bib}{}
\bibliographystyle{plainnat}

\newpage
\appendix

\section{Improved runtime of semi-definite relaxation by multiplicative weight update method} \label{sec:mw}

In this section, we show the multiplicative weight (MW)\ algorithm and
runtime analysis to solve the fair PCA relaxation in two groups for
\(n\times n\)  matrix up to \(\epsilon\) additive error in \(O(\frac{1}{\epsilon^2})\) iterations of solving a standard PCA, such as Singular Value Decomposition (SVD). Because SVD takes  \(O(n^3)\) time, the SDP relaxation \eqref{sdp} for two groups can be solved in \(O(\frac{n^3}{\epsilon^2})\). Comparing to \(O(n^{6.5}\log(1/\epsilon))\)  runtime of an SDP solver
that is commonly implemented with the interior point method
\citep{ben2001lectures}, our algorithm may be faster or slower depending on \(n,\epsilon\).
In practice, however, we tune the parameter of MW algorithm much more aggressively than in theory, and often take the last iterate solution of MW rather the average when the last iterate performs better, which gives a much faster convergence rate. Our runs of MW show that MW converges in at most 10-20 iterations. Therefore, we use MW to implement our fair PCA algorithm.
We note at the conclusion of this section that the algorithm and analysis can be extended to solving fair PCA in \(k\) groups up to additive error \(\epsilon\) in \(O(\frac{\log k}{\epsilon^2})\) iterations. 

Technically, the number of iterations for \(k\) groups is
\(O(\frac{W^2\log k}{\epsilon^2})\), where \(W\) is  the width
of the problem, as defined in \citet{arora2012multiplicative}. \(W\) can usually be bounded by the maximum number of
input or the optimal objective value. For our purpose, if the total variance  of input data over all dimension is \(L\), then the width \(W\) is at most \(L\). For simplicity, we assume \(L\leq 1\) (e.g. by normalization in prepossessing step), hence obtaining the \(O(\frac{\log k}{\epsilon^2})\)  bound on number of iterations.

We first present an algorithmic framework and the corresponding analysis in the next two subsections, and later apply those results to our specific setting of solving the SDP \eqref{sdp} from fair PCA problem. The previous work by  \citet{arora2012multiplicative} shows how we may solve a feasibility problem of an LP using MW technique. Our main theoretical contribution is to propose and analyze the optimization counterpart of the feasibility problem, and the MW algorithm we need to solve such problem. The MW we develop fits more seamlessly into our fair PCA setting and simplifies the algorithm to be implemented for solving the SDP \eqref{sdp}.

\subsection{Problem setup and oracle access}
We first formulate the feasibility problem and its optimization counterpart in this section. The previous and new MW algorithms and their analysis are presented in the following Section \ref{sec:MW-alg-ana}.
\subsubsection{Previous work: multiplicative weight on feasibility problem}
\paragraph{Problem} As in \citet{arora2012multiplicative}, we are given \(A\in \R^{m\times n}\) as an \(m\times n\) real matrix, \(x\in \R^n,b\in\R^m\), and \(\cP\) as a convex set in \(\R^n\), and the goal is to check the feasibility problem
\begin{equation}
\exists?  x\in \cP \ : \ Ax \geq b \label{feasibility}
\end{equation}
by giving a feasible \(x\in\cP\) or correctly deciding that such \(x\) does not exist.

\paragraph{Oracle Access} We assume the existence if an oracle that, given any probability vector \(p\in \Delta_m\) over \(m\) constraints of \eqref{feasibility}, correctly answers a single-constraint problem
\begin{equation}
\exists?  x\in \cP \ : \ p^\top Ax \geq p^\top b   \label{feasibility-oracle}
\end{equation}
by giving a feasible \(x\in\cP\) or correctly deciding that such \(x\) does not exist. We may think of \eqref{feasibility-oracle}  as a weighted version of \eqref{feasibility}, with weights on each constraint \(i\in[m]\) being \(p_i\).

As \eqref{feasibility-oracle} consists of only one constraint, solving \eqref{feasibility-oracle} is much easier than \eqref{feasibility} in many problem settings. For example, in our PCA setting, solving \eqref{sdp} directly is non-trivial, but the weighted version \eqref{feasibility-oracle} is a standard PCA problem: we weight each group \(A,B\) based on \(p\), and then apply a PCA algorithm (Singular Value Decomposition) on the sum of two weighted groups. The solution gives an optimal value of \(p^\top Ax-p^\top\ b\) in \eqref{feasibility-oracle}. More details of application in fair PCA settings are in Section \ref{sec:MW-fair-PCA}

\subsubsection{New setting: multiplicative weight on optimization problem}
\paragraph{Problem} The previous work gives an MW framework for the feasibility question. Here we propose an optimization framework, which asks for the best \(x\in\cP\) rather than an existence of \(x\in\cP\). The optimization framework can be formally stated as,  given \(A\in \R^{m\times n}\) as an \(m\times n\) real matrix, \(x\in \R^n,b\in\R^m\), and \(\cP\) as a convex set in \(\R^n\), we need to solve
\begin{equation}
\min z \ : \ Ax-b+z \cdot \onevec \geq0 , \text{ s.t. }x\in \cP  \label{optimize}
\end{equation}
where \(\onevec\) denotes the \(m\times 1\) vector with entries 1. Denote \(z^*\) the optimum of \eqref{optimize}.

With the same type of oracle access, we may run \eqref{feasibility} for \(O(\log \frac{n}{\epsilon})\) iterations to do binary search for the correct value of optimum \(z^*\) up to an additive error \(\epsilon\). However, our main contribution is to modify the previous multiplicative weight algorithm and the definition of the oracle to solve \eqref{optimize} without guessing the optimum \(z^*\). This improves the runtime slightly (reduce the \(\log(n/\epsilon)\) factor) and simplifies the algorithm.

 \paragraph{Feasibility Oracle Access} We assume the existence of an oracle that, given any probability vector \(p\in \Delta_m\) over \(m\) constraints of \eqref{optimize}, correctly answers a single-constraint problem
\begin{equation}
 \text{Find }  x\in \cP \ : \ p^\top Ax - p^\top b+z^* \geq 0   \label{optimize-oracle2}
\end{equation}

   There is always such \(x\) because multiplying \eqref{optimize} on the left by \(p^\top\) shows that one of such \(x\) is the optimum \(x^*\) of \eqref{optimize}. However, finding one may not be as trivial as asserting problem's feasibility. In general, \eqref{optimize-oracle2} can be tricky to solve since we do not yet know the value of \(z^*\).

\paragraph{Optimization Oracle Access} We define the oracle that, given \(p\in \Delta_m\) over \(m\) constraints of \eqref{optimize}, correctly answers one maximizer of
\begin{equation}
\min z \ : \ p^\top Ax-p^\top b+z \geq0 , \text{ s.t. }x\in \cP  \label{optimize-oracle}
 \end{equation}
which is stronger than and is  sufficient to solve \eqref{optimize-oracle2}. This is because \(x^*\) of \eqref{optimize} is one feasible \(x\) to \eqref{optimize-oracle}, so the optimum \(\hat{z}\) of \eqref{optimize-oracle} is at most \(z^*\). Therefore, the optimum \(x\) by \eqref{optimize-oracle} can be a feasible solution to \eqref{optimize-oracle2}. In many setting, because \eqref{optimize-oracle2} is only one-constraint problem, it is possible to solve the optimization version \eqref{optimize-oracle} instead. For example, in our fair PCA\ on two groups setting, we can solve the \eqref{optimize-oracle} by standard PCA on the union of two groups after an appropriate weighting on each group. More details of application in fair PCA settings are in Section \ref{sec:MW-fair-PCA}.
\subsection{Algorithm and Analysis} \label{sec:MW-alg-ana}
The line of proof follows similarly from  \citet{arora2012multiplicative}. We first state the technical property that the oracle satisfies in our optimization framework, then  show how to use that property to bound the number of iterations.
We fix \(A\in \R^{m\times n}\) as  an \(m\times n\) real matrix, \(x\in \R^n,b\in\R^m\), and \(\cP\) is a convex set in \(\R^n\)\begin{definition}
(analogous to \citet{arora2012multiplicative}) An \((\ell,\rho)\)-bounded \(\oracle\) for parameter \(0\leq \ell\leq \rho\) is an algorithm which, given  \(p\in\Delta_m\), solve \eqref{optimize-oracle2}. Also, there is a fixed \(I\subseteq[m]\) (i.e. fixed across all possible \(p\in\Delta_m\)) of constraints such that for all \(x\in \cP\) output by this algorithm,
\begin{align}
\forall i\in I: \ A_i x - b_i+z^* \in [-\ell,\rho] \\
\forall i\notin I: \ A_i x - b_i +z^*\in [-\rho,\ell]
\end{align}
\end{definition}
Note that even though we do not know \(z^*\), if we know the range of \(A_ix-b_i\) for all \(i\), we can  bound the range of \(z^*\). Therefore, we can still find a useful \(\ell,\rho\) that an oracle satisfies.

Now we are ready to state the main result of this section: that we may solve the optimization version by multiplicative update as quickly as solving the feasibility version of the problem.
\begin{theorem} \label{thm:mw-main}
Let \(\epsilon>0\) be given. Suppose there exists \((\ell,\rho)\)-bounded \(\oracle\) and \(\ell\geq\epsilon/4\) to solving \eqref{optimize-oracle2}. Then there exists an algorithm that solves \eqref{optimize} up to additive error \(\epsilon\), i.e. outputs \(x\in \cP\) such that \begin{equation}
Ax-b+z^* \cdot \onevec \geq-\epsilon
\end{equation}
The algorithm calls \(\oracle\) \(O(\ell\rho\log(m)/\epsilon^2)\) times and has additional \(O(m)\) time per call.
\end{theorem}
\begin{proof}
The proof follows similarly as Theorem 3.3 in \citet{arora2012multiplicative}, but we include details here for completeness. The algorithm is multiplicative update in nature, as in equation (2.1) of \citet{arora2012multiplicative}. The algorithm starts with uniform \(p^0\in\Delta_m\) over \(m\) constraints. Each step the algorithm asks the \(\oracle\) with input \(p^t\) and receive \(x^t\in\cP\). We use the loss vector \(m^t=\frac{1}{\rho}(Ax^t-b)\) to update the weight \(p^t\) for the next step with learning rate \(\eta\). After \(T\) iterations (which will be specified later), the algorithm outputs \(\bar{x}=\frac{1}{T}\sum_{t=1}^T x^t\).

 Note that using either the loss \(\frac{1}{\rho}(Ax^t-b+z^*)\)  and \(\frac{1}{\rho}(Ax^t-b)\) behaves the same algorithmically due to the renormalization  step on the vector \((p_i^t)_{i=1}^m\). Therefore, just for analysis, we use a hypothetical loss \(m^t=\frac{1}{\rho}(Ax^t-b+z^*)\) to update \(p^t\) (this loss can't be used  algorithmically since we do not know \(z^*\)). By Theorem 2.1 in   \citet{arora2012multiplicative}, for each constraint \(i\in[m]\) and all \(\eta\leq1/2\),
\begin{align}
\sum_{t=1}^T m^t \cdot p^t &\leq \sum_{t=1}^T m_i^t + \eta\sum_{t=1}^T |m_i^t|+\frac{\log m}{\eta} \nonumber \\
&= \frac{1}{\rho}\sum_{t=1}^T (A_i x^t - b_i+z^*) +\frac{\eta}{\rho}\sum_{t=1}^T |A_i x^t - b_i+z^*|+\frac{\log m}{\eta} \label{eq:regret-ineq}
\end{align}
By property \eqref{optimize-oracle2} of the  \(\oracle\),
\begin{equation}
\sum_{t=1}^T m^t \cdot p^t =\frac{1}{\rho} \sum_{t=1}^T \left((p^t)^\top(Ax^t-b)+z^*\right) \geq0 \label{eq:regret-term}
\end{equation}
We now split into two cases. If \(i\in I\), then \eqref{eq:regret-ineq} and \eqref{eq:regret-term} imply
\begin{align*}
0&\leq \frac{1+\eta}{\rho}\sum_{t=1}^T (A_i x^t - b_i+z^*) +\frac{2\eta}{\rho}\sum_{t:A_i x^t - b_i<0}|A_i x^t - b_i+z^*|+\frac{\log m}{\eta} \\
&\leq\frac{1+\eta}{\rho}T(A_i\bar{x}-b_i+z^*) +\frac{2\eta}{\rho}T\ell+\frac{\log n}{\eta}
\end{align*}
Multiplying the last inequality by \(\frac{\rho}{T}\) and rearranging terms, we have
\begin{align}
0 \leq (1+\eta)(A_i\bar{x}-b_i+z^*) + 2\eta \ell +\frac{\rho\log m}{T\eta} \label{eq:tune-eta-1}
\end{align}
If \(i\notin I\), then  \eqref{eq:regret-ineq} and \eqref{eq:regret-term} imply
\begin{align*}
0&\leq \frac{1-\eta}{\rho}\sum_{t=1}^T (A_i x^t - b_i+z^*) +\frac{2\eta}{\rho}\sum_{t:A_i x^t - b_i>0}|A_i x^t - b_i+z^*|+\frac{\log m}{\eta} \\
&\leq\frac{1-\eta}{\rho}T(A_i\bar{x}-b_i) +\frac{2\eta}{\rho}T\ell+\frac{\log n}{\eta}
\end{align*}
Multiplying inequality by \(\frac{\rho}{T}\) and rearranging terms, we have
\begin{align}
0 \leq (1-\eta)(A_i\bar{x}-b_i+z^*) + 2\eta\ell+\frac{\rho\log m}{T\eta} \label{eq:tune-eta-2}
\end{align}
To use \eqref{eq:tune-eta-1} and \eqref{eq:tune-eta-2} to show that \(A_i\bar{x}-b_i+z^*\) is close to 0 simultaneously for two cases, pick \(\eta = \frac{\epsilon}{8\ell}\) (note that \(\eta\leq1/2\) by requiring \(\ell\geq \epsilon/4\), so we may apply Theorem 2.1 in   \citet{arora2012multiplicative}). Then for all \(T \geq\frac{4\rho\log (m)}{\epsilon\eta}=\frac{32\ell\rho\log(m)}{\epsilon^2}\), we have
\begin{equation}
2\eta \ell +\frac{\rho\log m}{T\eta} \leq \frac{\epsilon}{4} +\frac{\epsilon}{4}=\frac{\epsilon}{2}
\end{equation}
Hence, \eqref{eq:tune-eta-1} implies
\begin{equation}
0 \leq (1+\eta)(A_i\bar{x}-b_i+z^*) +\frac{\epsilon}{2} \Rightarrow A_i\bar{x}-b_i+z^* \geq -\frac{\epsilon}{2}
\end{equation}
and \eqref{eq:tune-eta-2} implies
\begin{equation}
0 \leq (1-\eta)(A_i\bar{x}-b_i+z^*) +\frac{\epsilon}{2} \Rightarrow A_i\bar{x}-b_i+z^* \geq -\epsilon
\end{equation}
using the fact that \(\eta \leq 1/2\).
\end{proof}

\subsection{Application of multiplicative update method to the fair PCA problem} \label{sec:MW-fair-PCA}
In this section, we apply MW results for solving LP to solve the SDP relaxation \eqref{sdp} of fair PCA.

\paragraph{LP formulation of fair PCA relaxation} The SDP relaxation \eqref{sdp} of fair PCA can be written in the form \eqref{optimize} as an LP with two constraints
\begin{align} \label{sdp-mw-1}
\min_{P\in\cP,z\in\R} &z \text { s.t. } \\
z&\geq \alpha - \frac{1}{m_1}\langle A^\top A,P\rangle \label{sdp-mw-2} \\
z&\geq \beta -\frac{1}{m_2} \langle B^\top B,P\rangle \label{sdp-mw-3}
\end{align}
for some constants \(\alpha,\beta\), where the feasible region of variables is over a set of PSD matrices:
\begin{equation}
\cP=\{M\in\R^{n\times n}:0\preceq M \preceq I,\text{tr}(M)\leq d\}
\end{equation}
We will apply the multiplicative weight algorithm to solve \eqref{sdp-mw-1}-\eqref{sdp-mw-3}.
\paragraph{Oracle Access}
First, we present an the oracle in Algorithm \ref{alg:oracle}, which is in the form \eqref{optimize-oracle} and therefore can be used to solve \eqref{optimize-oracle2}. As defined in \eqref{optimize-oracle}, the optimization oracle, given a weight vector \(p=(p_1,p_2)\in\Delta_2\), should be able to solve the LP with one weighted constraint obtained from weighting two constraints \eqref{sdp-mw-2} and \eqref{sdp-mw-3} by \(p\). However, because both constraints involve only dot products of same variable \(P\) with constant matrices \( A^\top A\) and \( B^\top B\), which are linear functions, the weighted constraint will involve the  dot product of the same variable \(P\) with weighted sum of those constant matrices \(\frac{p_1}{m_1}A^\top A +\frac{p_2}{m_2} B^\top B\).

\begin{algorithm} % enter the algorithm environment
\caption{Fair PCA oracle (oracle to Algorithm \ref{alg:mw})} % give the algorithm a caption
\label{alg:oracle} % and a label for \ref{} commands later in the document
    \SetKwInOut{Input}{Input}
    \SetKwInOut{Output}{Output}

    %\underline{function Euclid} $(a,b)$\;
    \Input{\(p=(p_1,p_2)\in\Delta_2\), \(\alpha,\beta\in\R,\ A\in\R^{m_1\times n},B\in\R^{m_2\times n}\)}
    \Output{\(\argmin\limits_{P,z_1,z_2}\ \ p_1z_1+p_2z_2 \text{, \ subject to}\) \\ \(z_1=\alpha -\frac{1}{m_1} \langle A^\top A,P\rangle\),\\ \(z_2=\beta - \frac{1}{m_2}\langle B^\top B,P\rangle\), \\\(P\in\cP=\{M\in\R^{n\times n}:0\preceq M \preceq I,\text{Tr}(M)\leq d\}\)}
%    Define a new data set \(C=[\sqrt{p_1}\cdot A;\sqrt{p_2}\cdot B]\)\;
%    \tcc{We need to maximize \(p_1\langle A^\top A,P\rangle +p_2\langle B^\top B,P\rangle=\langle C^\top C,P\rangle\), so perform PCA on \(C\)}
    Set \(V\in\R^{n \times d}\) to be the matrix with top \(d\) principles components of \(\frac{p_1}{m_1}A^\top A + \frac{p_2}{m_2} B^\top B\) as columns\;
    \Return \(P^*=VV^\top\)\ , \(z_1^*=\alpha -\frac{1}{m_1}  \langle A^\top A,P^*\rangle,z_2^*=\beta - \frac{1}{m_2}\langle B^\top B,P^*\rangle\)\;
\end{algorithm}

\paragraph{MW Algorithm} Our multiplicative weight update algorithm for solving fair PCA relaxation \eqref{sdp-mw-1}-\eqref{sdp-mw-3} is presented in Algorithm \ref{alg:mw}.
The algorithm follows exactly from the construction in Theorem \ref{thm:mw-main}. The runtime analysis of our MW Algorithm \ref{alg:mw} follows directly from the same theorem.
\begin{algorithm} % enter the algorithm environment
\caption{Multiplicative weight update for fair PCA} % give the algorithm a caption
\label{alg:mw} % and a label for \ref{} commands later in the document
    \SetKwInOut{Input}{Input}
    \SetKwInOut{Output}{Output}

    %\underline{function Euclid} $(a,b)$\;
    \Input{\(\alpha,\beta\in\R,\ A\in\R^{m_1\times n},B\in\R^{m_2\times n}\), \(\eta>0\), positive integer \(T\)}
    \Output{\(\argmin\limits_{P,z}\ \ z \text{, \ subject to}\) \\ \(z\geq\alpha - \frac{1}{m_1}  \langle A^\top A,P\rangle\),\\ \(z\geq\beta - \frac{1}{m_2}\langle B^\top B,P\rangle\), \\\(P\in\cP=\{M\in\R^{n\times n}:0\preceq M \preceq I,\text{Tr}(M)\leq d\}\)}

    Initialize \(p^0=(1/2,1/2)\)\;
    \For{\(t=1,\ldots,T\)}{
    \((P_t,m_1^t,m_2^t)\leftarrow \oracle(p^{t-1},\alpha,\beta,A,B)\)\;
    \(\hat{p}_i^t\leftarrow p_i^{t-1}e^{\eta m_i^t}\), for \(i=1,2\)\;
    \(p_i^t\leftarrow \hat{p}_i^t/(\hat{p}_1^t+\hat{p}_2^t)\), for \(i=1,2\)\;
    }
    \Return \(P^*=\frac{1}{T}\sum_{t=1}^T P_t\) , \(z^*=\max\{\alpha - \frac{1}{m_1}  \langle A^\top A,P^*\rangle,\beta- \frac{1}{m_2}\langle B^\top B,P^*\rangle\}\)
\end{algorithm}

\begin{corollary}
Let \(\epsilon>0\). Algorithm \ref{alg:mw} finds a near-optimal (up to additive error of \(\epsilon\)) solution \(P\) to  \eqref{sdp-mw-1}-\eqref{sdp-mw-3} in \(O\left(\frac{1}{\epsilon^2}\right)\) iterations of solving standard PCA, and therefore in \(O(\frac{n^3}{\epsilon^2})\) running time.
\end{corollary}
\begin{proof}
We first check that the oracle presented in Algorithm \ref{alg:oracle} satisfies \((\ell,\rho)\)-boundedness and find those parameters. We may normalize the data so that the variances of \( \frac{A^\top A}{m_1}\) and \(\frac{B^\top B}{m_2}\) are bounded by 1. Therefore, for any PSD matrix \(P\preceq I\), we have \(\frac{1}{m_1}\langle A^\top A,P\rangle\leq 1\). In addition, in the application to fair PCA setting, we have \(\alpha = \frac{ \|\widehat{A}\|_F^2}{m_1}\). Hence,  \(\frac{1}{m_1}\langle A^\top A,P\rangle\leq \alpha\) for any feasible \(P\in\cP=\{M\in\R^{n\times n}:0\preceq M \preceq I,\text{Tr}(M)\leq d\}\) by the definition of \( \widehat{A}\) (recall Definition \ref{def:fair-PCA}). Therefore,
\begin{equation}
0 \leq \alpha -\frac{1}{m_1}\langle A^\top A,P\rangle \leq 1 ,\forall P\in\cP
\end{equation}
and similarly \( \beta -\frac{1}{m_2}\langle B^\top B,P\rangle \in [0,1]\). Hence, the optimal solution of Algorithm \ref{alg:mw}  satisfies \(z^*\in[0,1]\). Therefore, the oracle is \((1,1)\)-bounded.

Next we analyze the runtime of Algorithm \ref{alg:mw}. By Theorem \ref{thm:mw-main}, Algorithm \ref{alg:mw} calls the oracle \(O(1/\epsilon^2)\) times. The bottleneck in an oracle call is solving PCA on the weighted sum of two groups, which takes \(O(n^3)\) time. The additional processing time to update the weight is negligible compared to this \(O(n^3)\) time for solving PCA.
\end{proof}

\paragraph{MW for More Than Two Groups} Algorithms \ref{alg:oracle} and \ref{alg:mw} can be naturally extended to \(k\) groups. Theorem  \ref{thm:mw-main} states that we need \(O(\frac{\log k}{\epsilon^2})\) calls to the oracle with additional \(O(k)\) time per call (to update the weight for each loss). In each call, we must compute the weighted sum of \(k\) matrices of dimension \(n \times n\), which takes \(O(kn^2)\) arithmetic operations and perform SVD. In natural settings, \(k\) is much smaller than \(n\), and hence the runtime \(O(n^3)\) of SVD in each oracle call will dominate.

\section{Proofs}
\label{app:proof}

\begin{proofof}{Lemma~\ref{lem:A}}
Since $\rank(U) \leq d$, $\rank(V) \leq d$ and thus $\rank(\left[ \begin{array}{c} A \\ B \end{array} \right] VV^T) \leq d$. We will first show that $loss(A, AVV^T) \leq loss(A, U_A)$.

\textbf{Step 1.} Since $\{v_1,\ldots,v_d\}$ is an orthonormal basis of row space of $U$,  for every row of $U_A$, we have that $(U_A)_i = c_iV^T$ for some $c_i \in \mathbb{R}^{1 \times d}$.

\textbf{Step 2.} We show if we $c_i \rightarrow A_iV$ and consequently substitute the row $(U_A)_i \rightarrow A_i VV^T$, the value of $\|A_i - (U_A)_i\|$ decreases.
$$
\|A_i - (U_A)_i \|^2 = \|A_i - c_i V^T\|^2 = A_i {A_i}^T - 2A_i V {c_i}^T + c_i {c_i}^T
$$
Here we used the fact that $V^TV=I$. Minimizing the right hand side with respect to $c_i$, we get that $c_i = A_i V$.

\textbf{Step 3. }Step 2 proved that for every $i$, $\|A_i - A_i VV^T \|^2 \leq \|A_i - (U_A)_i \|^2$. Remember that

\begin{align*}
loss(A, U_A)  = \|A-U_A\|_F^2 - \|A-\widehat{A}\|_F^2 &= \sum \|A_i - (U_A)_i \|^2 - \|A-\widehat{A}\|_F^2    \\
loss(A, AVV^T)  = \|A-AVV^T\|_F^2 - \|A-\widehat{A}\|_F^2 &= \sum \|A_i - A_iVV^T \|^2 - \|A-\widehat{A}\|_F^2
\end{align*}
This finished the proof that $loss(A, AVV^T) \leq loss(A, U_A)$.  Similarly, we can see that $loss(B, BVV^T) \leq loss(B, U_B)$. Therefore
\begin{align*}
        f(\left[ \begin{array}{c} A \\ B \end{array} \right] VV^T) & = \max \big( \frac{1}{|A|} loss(A,AVV^T),  \frac{1}{|B|}  loss (B, BVV^T)\big)  \\
        & \leq \max \big( \frac{1}{|A|} loss(A,U_A),  \frac{1}{|B|}  loss (B, U_B)\big)\\
        &= f(U)
\end{align*}
\end{proofof}

\begin{proofof}{Lemma~\ref{lem:B}}
From Lemma~\ref{lem:optPCA}, we know that there exist a matrix $W_A \in \mathbb{R}^{n\times d}$ such that $W_A^T W_A = I$ and $\widehat{A} = A W_A W_A^T$. Considering this and the fact that $V^TV=I$
\begin{align*}
        loss(A, AVV^T) &= \|A-AVV^T\|_F^2 - \|A-AW_AW_A^T\|_F^2 \\
        &= \sum_{i} \|A_i - A_i VV^T\|^2 - \|A_i - A_i W_A W_A^T\|^2 \\
        & = \sum_i A_i A_i^T -  A_i VV^T A_i^T - (\sum_i A_i A_i^T - \sum_i A_i W_A W_A^T)\\
        & =\sum_{i} A_i W_A W_A^T A_i^T - \sum_{i} A_i VV^T A_i^T \\
        \sum_{i} A_i W_A W_A^T A_i^T &= \sum_i \|A_i W_A\|^2 = \|AW_A\|_F^2 = \|AW_A W_A^T\|_F^2 = \|\widehat{A}\|_F^2 \\
        \sum_{i} A_i VV^T A_i^T &= \sum_i \|A_i V\|^2 = \|AV\|_F^2 = \sum_i \|Av_i\|^2 \\
        \sum_{i} A_i VV^T A_i^T &= \sum_i \|A_i V\|^2 = \|AV\|_F^2 = \Tr(V^TA^TAV) = \Tr(VV^TA^TA) = \langle A^TA, VV^T\rangle
\end{align*}
Therefore $loss(A, AVV^T) = \|\widehat{A}\|_F^2 - \sum_{i=1}^{d} \|Av_i\|^2 = \|\widehat{A}\|_F^2 - \langle A^TA, VV^T\rangle$.

\begin{align*}
        \|A - AVV^T\|_F^2 &= \sum_i \|A_i - A_iVV^T\|^2 = \sum_i A_i A_i^T - \sum_i A_iVV^TA_i^T\\
        &= \|A\|_F^2 - \sum_i \|Av_i\|^2 = \|A\|_F^2 - \|AV\|_F^2 \\
\end{align*}

\end{proofof}

\begin{proofof}{Lemma~\ref{lem:C}}

We prove that the value of function $g_A$ at its local minima is equal to its value at its global minimum, which we know is the subspace spanned by a top $d$ eigenvectors of $A^TA$. More precisely, we prove the following: Let $\{v_1, \ldots, v_n\}$ be an orthonormal basis of eigenvectors of
$A^TA$ with corresponding eigenvalues
$\lambda_1 \geq \lambda_2 \geq \ldots \geq \lambda_n$ where ties are
broken arbitrarily. Let $V^*$ be the subspace spanned by
$\{v_1,\ldots, v_d\}$ and let $U$ be some $d$-dimensional subspace
s.t. $g_A(U) > g_A(V^*)$. There is a continuous path from $U$ to $V^*$
s.t. the value of $g_A$
is monotonically decreasing for every $d$-dimensional subspace on the path.

Before starting the proof, we will make a couple of notes which would be used throughout the proof. First note that $g_A(V)$ is well-defined i.e., the value of $g_A(V)$ is only a function of the subspace $V$. More precisely, $g_A(V)$ is invariant with respect to different choices of orthonormal basis of the subspace $V$. Second, given Lemma~\ref{lem:B}, $g_A(V) = \|A\|_F^2 - \sum_{i}\|Av_i\|^2$. Therefore, proving that $g_A(V)$ is decreasing is equivalent to proving that $\sum_i \|Av_i\|^2$ is increasing as a function of any choice of orthonormal basis of the subspaces on the path.

$g_A(U) > g_A(V^*)$ therefore $U\neq V^*$. Let $k$ be the smallest index such that $v_k \notin U$. Extend $\{v_1,\ldots, v_{k-1}\}$ to an orthonormal basis of $U$: $\{v_1, \ldots, v_{k-1}, v'_k, \ldots, v'_d\}$. Let $q \geq k$ be the smallest index such that $\|Av_q\|^2 > \|Av'_q\|^2$. Such an index $q$ must exist given that $g_A(U)> g_A(V^*)$. Without loss of generality we can assume that $q=1$ (this will be clear throughout the proof). Therefore, we assume that $v_1$, the top eigenvector of $A^TA$, is not in $U$ and that it strictly maximizes the function $\|Au\|^2$ over the space of unit vectors $u$. Specifically, for any unit vector $u\in U$, $\|Au\|^2<  \|Av_1\|^2 = \lambda_1$.

Let $v_1 = \sqrt{1-a^2}z_1 + az_2$ where $z_1 \in U$ and $z_2 \perp U$, $\|z_1\|=\|z_2\|=1$ i.e., the projection of $v_1$ to $U$ is $\sqrt{1-a^2}z_1$. We distinguish two cases:

\paragraph{Case \boldmath $z_1 = 0$.} $v_1\perp U$. Let $w = \sqrt{1-\epsilon^2} u_1 + \epsilon v_1$. $\|w\|=1$. Note that $\{w, u_2, \ldots, u_d\}$ is an orthonormal set of vectors. We set $U_{\epsilon}= span\{w, u_2, \ldots, u_d\}$. We show that $g_A(U_\epsilon) < g_A (U)$. Using the formulation of $g$ from Lemma~\ref{lem:B}, we need to show that $\|Aw\|^2 + \|Au_2\|^2+\ldots+\|Au_d\|^2 > \|Au_1\|^2+\|Au_2\|^2+\ldots + \|Au_d\|^2$ or equivalently that $\|Aw\|^2 > \|Au_1\|^2$.

\begin{align*}
\|Aw\|^2 - \|Au_1\|^2 &= \|A(\sqrt{1-\epsilon^2}u_1+\epsilon v_1)\|^2 - \|Au_1\|^2\\
&=(\sqrt{1-\epsilon^2}u_1^T+\epsilon v_1^T)A^TA(\sqrt{1-\epsilon^2}u_1+\epsilon v_1) - \|Au_1\|^2\\
&=(1-\epsilon^2)u_1^TA^TAu_1+\epsilon^2v_1^TA^TAv_1+2\sqrt{1-\epsilon^2}\epsilon u_1^TA^TAv_1 - \|Au_1\|^2\\
&=(1-\epsilon^2)\|Au_1\|^2+\epsilon^2\lambda_1+2\epsilon \sqrt{1-\epsilon^2} u_1^TA^TAv_1 - \|Au_1\|^2 \\
&= \epsilon^2 (\lambda_1 - \|Au_1\|^2) + 2\epsilon \sqrt{1-\epsilon^2} u_1^TA^TAv_1 \\
\end{align*}
where $ u_1^TA^TAv_1 = u_1^T (\lambda_1 v_1)= \lambda_1 u_1^Tv_1 =0$ since $v_1$ is an eigenvector of $A^TA$ and $v_1 \perp u_1$. This, and considering the fact that $\|Au_1\|^2 < \lambda_1$
\begin{align*}
        \|Aw\|^2 - \|Au_1\|^2 &= \epsilon^2(\lambda_1 - \|Au_1\|^2) > 0
\end{align*}
Therefore, $\|Aw\|^2 > \|Au_1\|^2$ and thus $g_A(U_\epsilon) < g_A (U)$.

\paragraph{Case \boldmath $z_1 \neq 0$}. Note that $z_2 \neq 0$ either since we picked $v_1 \notin U$. Let's extend $\{z_1\}$ to an orthonormal basis of $U$:  $\{z_1,u_2, \ldots, u_k\}$. We will transform $U$ s.t. the resulting subspace $U_1$ is the span of $v_1, u_2,\ldots, u_k$. This can then be repeated orthogonal to $v_1$ till the subspace becomes $V^*$.

For small enough  $\epsilon > 0$, consider the unit vector $w = \sqrt{1-\epsilon^2}z_1 + \epsilon z_2$. We will move $U$ to $U_\epsilon:=span \{w,u_2,\ldots,u_d\}$. The latter is an orthonormal representation since both $z_1$ and $z_2$ are orthogonal to all of $u_2, \ldots, u_d$  and $w$ is in the span of $z_1, z_2$. We will prove that $g_A(U_\epsilon) < g_A(U)$. Given Lemma~\ref{lem:B}, since the chosen orthonormal basis of these two subspaces differ only in $w$ and $z_1$, it suffices to show that $\|Aw\|^2 > \|Az_1\|^2$.
We can write
\begin{align*}
w &= \left(\sqrt{1-\epsilon^2} - \frac{\epsilon\sqrt{1-a^2}}{a}\right)z_1 +  \frac{\epsilon}{a}(\sqrt{1-a^2}z_1 + az_2 )=\left(\sqrt{1-\epsilon^2} - \frac{\epsilon\sqrt{1-a^2}}{a}\right)z_1 +  \frac{\epsilon}{a}v_1.
\end{align*}
Therefore, noting that $A^TAv_1 = \lambda_1 v_1$ since $v_1$ is an eigenvector with eigenvalue $\lambda_1$ and $z_1^Tv_1=\sqrt{1-a^2}$,
\begin{align*}
&\|Aw\|^2 \\
&= \left( \sqrt{1-\epsilon^2} - \frac{\epsilon\sqrt{1-a^2}}{a} \right)^2\|Az_1\|^2 + \frac{\epsilon^2}{a^2} \|Av_1\|^2 + 2\frac{\epsilon}{a}\left(\sqrt{1-\epsilon^2} - \frac{\epsilon\sqrt{1-a^2}}{a}\right)z_1^TA^TAv_1\\
&=\left(1-\epsilon^2 + \frac{\epsilon^2(1-a^2)}{a^2}-2\frac{\epsilon\sqrt{(1-\epsilon^2)(1-a^2)}}{a}\right)\|Az_1\|^2 + \frac{\epsilon^2}{a^2}\lambda_1 \\
& + 2\frac{\epsilon}{a}\left(\sqrt{1-\epsilon^2} - \frac{\epsilon\sqrt{1-a^2}}{a}\right)\lambda_1 z_1^Tv_1
=\left(1-2\epsilon^2 + \frac{\epsilon^2}{a^2}-2\frac{\epsilon\sqrt{(1-\epsilon^2)(1-a^2)}}{a}\right)\|Az_1\|^2 \\
&+ \left(\frac{\epsilon^2}{a^2} + 2\frac{\epsilon\sqrt{(1-\epsilon^2)(1-a^2)}}{a} - 2\frac{\epsilon^2(1-a^2)}{a^2}\right)\lambda_1 \\
&=\|Az_1\|^2 + (\lambda_1-\|Az_1\|^2)\left(2\frac{\epsilon\sqrt{(1-\epsilon^2)(1-a^2)}}{a}+2\epsilon^2 - \frac{\epsilon^2}{a^2}\right)\\
&>\|Az_1\|^2
\end{align*}
where the last inequality follows since $\lambda_1 > \|Az_1\|^2$ and we can choose $ 0 < \epsilon < \frac{1}{1+C}$ for $C= 4a^2(1-a^2)$ so that $ 2\frac{\epsilon\sqrt{(1-\epsilon^2)(1-a^2)}}{a} > \frac{\epsilon^2}{a^2}$ .
Thus, $\|Aw\|^2 > \|Az_1\|^2$ and therefore $g_A(U_\epsilon) < g_A(U)$.
proving the claim.

\end{proofof}

\end{document}